\pdfoutput=1

\documentclass[12pt]{article}
\usepackage{booktabs} 
\usepackage[utf8]{inputenc} 
\usepackage[T1]{fontenc}    
\usepackage{hyperref}       
\usepackage{url}            
\usepackage{booktabs}       
\usepackage{amsfonts}       
\usepackage{nicefrac}       
\usepackage{microtype}      
\usepackage{amsmath,amstext,amsthm,amsfonts,latexsym,amssymb,bbm}
\usepackage{array}
\usepackage{algorithm}
\usepackage{algorithmic}
\usepackage{xcolor} 

\newtheorem{proposition}{Proposition}
\newtheorem*{proposition*}{Proposition}
\newtheorem{theorem}{Theorem}

\newtheorem{fact}{Fact}

\newcommand{\PP}{\mathbb{P}}
\newcommand{\EE}{\mathbb{E}}
\DeclareMathOperator{\tr}{tr}

\DeclareMathOperator{\indic}{\mathbbm{1}}
\newcommand\rev[1]{{\color{black} #1}}
\usepackage{hyperref}
\makeatletter
\DeclareRobustCommand*\cal{\@fontswitch\relax\mathcal}
\makeatother
\def\BibTeX{{\rm B\kern-.05em{\sc i\kern-.025em b}\kern-.08emT\kern-.1667em\lower.7ex\hbox{E}\kern-.125emX}}
\begin{document}
\title{Computationally Efficient Estimation of the Spectral Gap of a Markov Chain
\author{Richard Combes ($\dag$) and Mikael Touati ($\times$)}
\thanks{$\dag$: Centrale-Supelec, Gif-sur-Yvette, France}
\thanks{$\times$: Orange Labs, Paris, France}} 
\renewcommand\footnotemark{}
\renewcommand\footnoterule{}
\maketitle
\begin{abstract}
	We consider the problem of estimating from sample paths the absolute spectral gap $1 - \lambda_\star$ of a reversible, irreducible and aperiodic Markov chain $(X_t)_{t \in \mathbb{N}}$ over a finite state space $\Omega$. We propose the ${\tt UCPI}$ (Upper Confidence Power Iteration) algorithm for this problem, a low-complexity algorithm which estimates the spectral gap in time ${\cal O}(n)$ and memory space ${\cal O}((\ln n)^2)$ given $n$ samples. This is in stark contrast with most known methods which require at least memory space ${\cal O}(|\Omega|)$, so that they cannot be applied to large state spaces. \rev{We also analyze how $n$ should scale to reach a target estimation error.} Furthermore, ${\tt UCPI}$ is amenable to parallel implementation.
\end{abstract}
\section{Introduction and related work}\label{sec:introduction}

\subsection{Introduction and contribution}
We consider the problem of estimating from sample paths the absolute spectral gap $1 - \lambda_\star$ of a reversible, irreducible and aperiodic Markov chain $(X_t)_{t \in \mathbb{N}}$ over a finite state $\Omega$. The absolute spectral gap is $1 - \lambda_\star$ where $\lambda_\star$ is the second largest (in absolute value) eigenvalue of $P$ , the transition matrix of the chain. We are interested in the case where $P$ is not known explicitly to the experimenter. Rather, she may draw sample paths from the Markov chain, and solely use those sample paths in order to estimate $\lambda_\star$. In other words we are interested in simulation-based (or black-box) estimation. Furthermore, we mainly concern ourselves with the case where the state space $\Omega$ is large since this is often the case in practical applications. In fact, when $\Omega$ is large (even if $P$ is known) one may have to resort to a simulation-based method since storing $P$ and manipulating it requires at least ${\cal O}(|\Omega|)$ memory space if $P$ is sparse and ${\cal O}(|\Omega|^2)$ in the worse case. Therefore, we focus on simulation-based, computationally efficient estimation of the absolute spectral gap $1 - \lambda_\star$ of Markov chains over large state spaces.

\subsubsection{Our contribution} We make the following contributions (a) We propose ${\tt UCPI}$ (Upper Confidence Power Iteration), a computationally efficient algorithm to estimate the absolute spectral gap in time ${\cal O}(n)$ and memory space ${\cal O}((\ln n)^2)$ given $n$ samples. \rev{We analyze how $n$ should scale to reach a target estimation error in section~\ref{subsec:complexity}.} (b) We prove that ${\tt UCPI}$ is consistent and analyze its convergence rate as a function of the number of samples. (c) We show how ${\tt UCPI}$ is applicable to a broad set of assumptions e.g. the case where a single sample path is available, the case where one can simulate transitions of the chain etc.

\subsection{Markov chains over large state spaces}  Markov chains are typically encountered in the following setting. Consider $\pi$ the stationary distribution of $P$ and  $f:\Omega \to [0,1]$ some function. We would like to compute the quantity: 
$$
Z = \sum_{x \in \Omega} f(x) \pi(x)  = \EE_{X \sim\pi} f(X).
$$
If $\Omega$ is large, even if $\pi$ were known, summing over all elements of $x \in \Omega$ might not be feasible and one may instead use a simulation method by drawing sample paths of $(X_t)_{t}$ starting from some arbitrary initial distribution, since:
$$
	\EE( f(X_k) ) \underset{k \to \infty}{\to} Z. 
$$
So we may compute $Z$ by drawing many independent sample paths of length $k$ with $k$ large enough. However, to select the sample path length $k$ properly, one needs information about the mixing properties (i.e. the convergence speed in the above expression) of $(X_t)_{t \in \mathbb{N}}$. Very often, in practice, the sample path length is chosen by hand using some heuristic rule. Therefore estimating the spectral gap and other measures of the mixing rate of $(X_t)_{t \in \mathbb{N}}$ has great practical importance since it enables to select the sample path length automatically using a principled approach. Some examples of this setting are (a) Markov Chain Monte Carlo (MCMC), to draw from a stationary distribution $\pi$ known up to normalization \cite{AndrieuMCMC} (b) in statistical physics \cite{MezardMontanari} to compute the partition function of spin models (c) in Bayesian statistics, to compute the posterior distribution, which is usually intractable and (d) to compute the stationary performance of systems described by complex queuing models. Surprisingly, in some problems such as computing the volume of a high dimensional convex body, MCMC is the only known method which is provably efficient \cite{LovaszMCMC}.

\subsection{Mixing properties of Markov chains}

 We give a quick overview of several measures for quantifying mixing of Markov Chains (refer to \cite{Lev2009} for a complete survey). The mixing time $t_{m}$ is the smallest $t$ such that:
$$
	{1 \over 2} \max_{\mu} \sum_{x \in \Omega} \Big|\PP(X_t = x | X_0 \sim \mu) - \pi(x)\Big| \le {1 \over 4}
$$
where the maximum is taken over all the possible initial distributions $\mu$. Another measure is the relaxation time $t_{r} = {1 \over 1 - \lambda_\star}$ which verifies for reversible, irreducible and aperiodic chains:
$$
	(t_{r} - 1) \ln 2 \le t_{m} \le t_{r} \ln {4 \over \min_{x \in \Omega} \pi(x)}.
$$
Other measures include the conductance which is linked to the spectral gap through Cheeger's inequality. Some of these measures are hard to compute, for instance computing conductance is in general NP-hard \cite{Lov1996}. We focus on estimating $\lambda_\star$, which enables to estimate $t_{r}$.

In fact, for any initial distribution $X_0 \sim \mu$ and any function $f:\Omega \to [0,1]$, the bias of the MCMC method can be controlled with the following inequality (see \cite{Lev2009}):
$$
	\Big|Z - \EE( f(X_k) ) \Big| \le 2^{ -\lfloor {k \over t_{m}} \rfloor} \;,\; \forall k
$$
If the experimenter knows $t_m$ or $t_r$, she may select the sample path length $k$ to ensure that the bias is smaller than any specified target error. Since the bias \emph{cannot} be directly estimated, being able to access $t_r$ or $t_m$ is the only way one can choose the sample path length $k$ in a principled manner, and ensure that the result of MCMC is correct up to some given accuracy. It is noted that while knowing an \emph{upper} bound of $t_m$ or $t_r$ is very interesting in practice (it enables to ensure correctness of the MCMC method), knowing a lower bound or a point estimate of $t_m$ or $t_r$ seems much less interesting, since it does not enable to guarantee anything about the output of the MCMC method.

\subsection{Objectives}

	\subsubsection{Computational complexity.} Since $\lambda_\star$ is a function of the transition matrix $P$, the most natural approach to estimating $\lambda_\star$ would be to first estimate $P$ by:
$$
{\hat P}(x,y) = { \sum_{t=1}^n \indic\{ (X_t,X_{t+1}) = (x,y)\} \over \sum_{t=1}^n \indic\{ X_t = x \}}
$$
the empirical conditional distribution of $X_{t+1}$ knowing $X_t$, and then compute the eigenvalues of $\hat P$. However this approach might not be feasible for large state spaces, as simply storing $\hat P$ requires ${\cal O}(|\Omega|^2)$ space and computing its eigenvalues requires time ${\cal O}(|\Omega|^3)$.  \rev{We are interested in methods that are computationally efficient (requiring less than ${\cal O}(|\Omega|)$ memory space) and can be used on large state spaces}. \rev{Also, the approach described above makes it difficult to derive tight confidence intervals on $\lambda_\star$ - as the variance of the estimate is rather involved. Thus obtaining estimates whose variations are easier to control as done in the present work is highly desirable}.

	\subsubsection{Upper confidence bounds} The main goal of estimating $\lambda_\star$ is to estimate $t_{r}$, which enables to select the sample path length in MCMC methods. As explained above, point estimates and lower confidence bounds are of little interest, and we are mostly interested in \emph{upper} confidence bounds for $t_r$, and this requires \emph{upper} confidence bounds on $\lambda_\star$. Namely if one has an estimate ${\hat \ell_\star}$ such that ${\hat \ell_\star} \ge \lambda_\star$ with high probability, then ${1 \over 1 - {\hat \ell_\star}} \ge t_r$ with high probability as well and ${\hat \ell_\star}$ is sufficient to control the error of MCMC.
	
	\subsubsection{Black box methods} In many problems of interest, one may not access elements of the transition probability matrix directly, and one may only simulate transitions of the Markov chain, so that one can only estimate $\lambda_\star$ by observing sample paths of the said Markov chain. This setting is interesting even in the case where $P$ is known but too large to be stored in memory, so that we may use simulation methods as a way to get around the size of $P$. 

	\subsubsection{Goal} Our goal is therefore to find upper confidence bounds of $\lambda_\star$ in a computationally efficient manner, when one can only simulate the Markov chain at hand.

\subsection{Related work}

 \subsubsection{Non black-box methods} If $P$ is known, computing the spectral gap can be done using several methods. A naive approach is to compute the spectrum of $P$ which requires ${\cal O}(|\Omega|^3)$ time and ${\cal O}(|\Omega|^2)$ space in the worse case. To compute the largest eigenvalue, a more efficient method is Power Iteration which requires  ${\cal O}(|\Omega|^2)$ time and space in general, and ${\cal O}(|\Omega|)$ time and space if $P$ is sparse \cite{Qua2007}. Applying Power Iteration several times enables to compute eigenvalues in deceasing order of magnitude. However, if $\Omega$ is very large these methods are not usable and of course, they do not apply to the black box setting.

\subsubsection{Black-box methods} A contribution closely related to ours is \cite{Hsu2015}, where the authors provide both estimates and confidence intervals for mixing times in the black-box setting where only one sample path is available. From this sample path, they construct the empirical transition matrix ${\hat P}$ and use it as a plugin estimator for $P$, and they prove that this is statistically efficient. Their method requires time ${\cal O}(n+|\Omega|^3)$ and space ${\cal O}(|\Omega|^2)$, and is hence mainly adapted to state spaces of moderate size. Our contribution is therefore complementary to \cite{Hsu2015}, which is  mainly concerned with statistical efficiency for moderate $\Omega$ while we are interested computationally efficient approaches for large $\Omega$.

\subsubsection{Point estimates} A recent contribution \cite{Qin2017} does provide point estimates for $\lambda_{\star}$ for Markov chains on general, infinite state spaces, however no confidence intervals are provided, so that one cannot control the error of MCMC methods using the proposed approaches. 

\subsubsection{Perfect sampling} Finally, there exists perfect sampling methods which allow to sample from the stationary distribution by considering random sampling times such as \cite{ProppWilson}. This removes the need to study the mixing of $P$. However these methods cannot be applied to large state spaces as they imply simulating ${\cal O}(|\Omega|)$ Markov chains in parallel, unless the trajectories of the Markov chain have some special monotonicity properties. \rev{Therefore the memory required for perfect sampling is at least ${\cal O}(|\Omega|)$, while we focus on more memory-efficient approaches.}

\rev{
\subsubsection{Spectral analysis from noisy observations}  There exists a growing litterature on how to estimate eigenvalues of a matrix $P$, which may or may not be a transition matrix, when one may only sample entries of $P$ \cite{Han2017,Kong2017,Khetan2017}. While these methods are more general than ours and could be applied to our setting, we leverage the fact that for Markov chains one may sample from $P^k$ by sampling $k \ge 1$ successive transitions.}

\rev{
\subsubsection{Computational complexity} In summary all techniques from the literature require at the very least time and space ${\cal O}(|\Omega|)$. We focus on  more efficient approaches in terms of time and memory, in particular we propose methods that require strictly less memory space than ${\cal O}(|\Omega|)$.}

The remainder of the paper is organized as follows. In section \ref{sec:model} we state our model and in section~\ref{sec:algorithm} we state the ${\tt UCPI}$ algorithm and provide its computational complexity. In section \ref{sec:genericAnalysis} we analyze the estimation error of ${\tt UCPI}$ as a function of the number of samples, and explain how to set its input parameters. In section \ref{sec:extensions}, we generalize our results to a lighter set of assumptions. In section \ref{sec:numerical_experiments}, we perform numerical experiments and section~\ref{sec:conclusion} concludes the paper.
\section{Model}\label{sec:model}
	\subsection{Definitions} 
	We consider a time-homogeneous Markov chain $(X_t)_{t \in \mathbb{N}}$ on a finite state space $\Omega$ with transition matrix $P$. We write 
$$
P(x,y)=\PP(X_{t+1} = y | X_t = x) \,,\, x,y \in \Omega,
$$
 the transition probability from state $x$ to state $y$. We assume that $(X_t)_{t \in \mathbb{N}}$ is irreducible, aperiodic (thus ergodic with unique stationary distribution $\pi$), reversible and lazy i.e. 
 $$
 P(x,x) \ge {1 \over 2} \,,\,x \in \Omega.
 $$
  Since $P$ is both reversible and lazy, its eigenvalues are real and positive.   We denote by $(\lambda_1,\dots,\lambda_{|\Omega|})$ the eigenvalues of $P$, which we assume sorted so that:
$$
1 = \lambda_{1} > \lambda_{2} \geq \lambda_3 \ge \lambda_{|\Omega|} \ge 0,
$$
where $\lambda_1 = 1$ and $1 > \lambda_2$ since the chain is irreducible and aperiodic. The spectral gap of $P$ is $1-\lambda_2$. Throughout the paper, we have $$
\lambda_2 = \lambda_\star := \max \left\{ |\lambda|: \lambda \text{ eigenvalue of } P , \lambda\neq 1 \right\}.
$$

Even though the laziness assumption might seem restrictive at first sight,  we show that our results can easily be extended to all reversible chains (Section~\ref{sec:extensions}). 

\subsection{Computational models} 

In order to estimate $\lambda_\star$, we do not assume that the transition matrix $P$ is known explicitly, but we simply assume that one is able to draw sample paths $(X_t)_{t \in \mathbb{N}}$ in order to perform this estimation. Namely we are interested in simulation-based (or black-box) estimation of $\lambda_\star$. We will look at two possible computational models.

	\subsubsection{The Random Transition Function (RTF) model} Here one has access to two subroutines ${\tt NextState}$ and ${\tt Uniform}$, whose output is random. The ${\tt NextState}$ sub-routine takes one input argument, and, for any $x$, ${\tt NextState}(x)$ outputs a random variable whose distribution is that of $X_{t+1}$ knowing that $X_{t} = x$, namely:
	$$
	\PP({\tt NextState}(x) = y) = P(x,y)  \,,\, x,y \in \Omega.
	$$
	 In other words, ${\tt NextState}$ simulates a single transition of the Markov chain $(X_t)_{t \in \mathbb{N}}$ and one can simulate a sample path using repeated calls to ${\tt NextState}$. The ${\tt Uniform}$ subroutine takes no input argument, and its output is a random variable uniformly distributed on $\Omega$ i.e 
	 $$
	 		\PP({\tt Uniform}() = x) = {1 \over |\Omega|}  \,,\, x \in \Omega.
	 $$
	 It is noted that, in general, when $\Omega$ is a set with large size, drawing a sample from the uniform distribution may not be straightforward (this is more or less equivalent to counting the number of elements in $\Omega$). 
	 In section~\ref{sec:extensions}, we address the case where using the uniform distribution may not be feasible and a general distribution $\mu$ on $\Omega$ is used. We assume that a call to subroutine ${\tt NextState}$ (or ${\tt Uniform}$) requires ${\cal O}(1)$ time to complete. We assume that $n$ samples are available to estimate $\lambda_\star$, so that one is allowed (at most) $n$ calls to ${\tt NextState}$ before outputting an estimate.
	
	\subsubsection{The Unique Sample Path (USP) model} Here one is simply provided with a sample path $(X_t)_{t=0,\dots,n}$ of length $n$ in order to estimate $\lambda_\star$. One is not allowed to choose the starting state $X_0$. This is the model considered in~\cite{Hsu2015}. Furthermore, since a sample path of length $n$ can be generated by $n$ calls to ${\tt NextState}$, estimation in the USP model is more arduous than in the RTF model. While we mainly focus on the RTF model, we show in section~\ref{sec:extensions} how the USP model can be reduced to RTF, so that the same estimation techniques may be used. Furthermore, this reduction is computationally efficient (in time ${\cal O}(n)$), so that the computational complexity is the same in both models. 

\subsection{Models in practice} 
The RTF model applies whenever one has access to a black-box that is able to simulate one-step transitions from a Markov chain, but one cannot determine $P$ a priori (for instance because the underlying processes driving the transitions of $(X_t)_t$ are too complex). A typical example of this situation is Reinforcement Learning (RL): here $(X_t)_{t \in \mathbb{N}}$ describes the sequence of rewards obtained by applying some fixed policy, one estimates the average reward by averaging $(X_t)_{t \in \mathbb{N}}$ for ''long enough''. 

	The RTF model also applies for MCMC methods, where $P$ is known (it is in fact chosen by the experimenter), but the state space $\Omega$ is so large that one cannot store $P$ in memory, and one may only draw sample paths by simulating successive transitions. Since MCMC methods are usually designed to compute a summation over the elements of $\Omega$, they only make sense when one is not able to use ${\cal O}(|\Omega|)$ memory space, otherwise one could simply enumerate the elements of $\Omega$. More generally, many applications consider Markov chains over very large state spaces, where $P$ is known but storing it in memory is not feasible. For instance, for spin systems used in statistical physics (e.g Ising and Potts models), there are at least $|\Omega| \ge 2^{s}$ states where $s$ is the number of particles. In other words, in these situations one chooses to use a simulation-based estimate simply for computational reasons. The USP model is more relevant when $(X_t)_{t}$ represents historical data of some process that the experimenter cannot modify at will (e.g. climate data, financial data etc.). 
\section{The UCPI Algorithm}\label{sec:algorithm}
In this section we state ${\tt UCPI}$, a computationally efficient algorithm for estimating the spectral gap of a Markov Chain. Throughout the next sections we consider the RTF model. The extension of ${\tt UCPI}$ to the USP model is addressed in section~\ref{sec:extensions}.

\subsection{Retrieving the spectral gap from traces} Before describing ${\tt UCPI}$ in details, we highlight its rationale, which consists in estimating $\lambda_2$ based on the trace of the powers of $P$. This idea is in fact the backbone of Power Iteration. Recall that the eigenvalues of $P$ are $(\lambda_1,\dots,\lambda_{|\Omega|})$, so the eigenvalues of $P^k$ are $(\lambda_1^k,\dots,\lambda_{|\Omega|}^k)$. Since the trace of a matrix is both the sum of its diagonal elements and of its eigenvalues: 
$$
	 \sum_{x \in \Omega} P^k(x,x) =  \tr(P^k) = \sum_{i=1}^{|\Omega|} \lambda_i^{k}.  
$$
Since $\lambda_1 = \lambda_1^k = 1$, we have:
$$
	 \Big[\sum_{x \in \Omega} P^k(x,x) - 1\Big]^{1 \over k} = \lambda_2\Big[\sum_{i=2}^{|\Omega|} \Big({\lambda_i \over \lambda_2}\Big)^k \Big]^{1 \over k}. 
$$
Since $\lambda_2 \ge \lambda_i \ge 0$ for $i=2,\dots,|\Omega|$, we get:
$$
	1 \le \Big[\sum_{i=2}^{|\Omega|} \Big({\lambda_i \over \lambda_2}\Big)^k \Big]^{1 \over k} \le |\Omega|^{1 \over k} \underset{k \to \infty}{\to} 1,
$$
and letting $k \to \infty$ above yields:
\begin{equation}\label{eq:approximation_convergence}
	\lim_{k \to \infty} \Big[\sum_{x \in \Omega} P^k(x,x) - 1\Big]^{1 \over k} = \lambda_2.
\end{equation}

\subsection{Trace estimates} The relationship above shows that one can estimate $\lambda_2$ arbitrary well by estimating the trace of $P^k$ for $k$ arbitrarly large. The cornerstone of ${\tt UCPI}$ is that the trace of $P^k$ can be easily estimated from sample paths. Indeed, it suffices to estimate the probability that the chain $(X_t)_t$ returns to its starting state after $k$ steps, providing that its starting state is chosen uniformly at random. Consider $U$ uniformly distributed on $\Omega$, then:
\begin{align}\label{eq:trace_proba_equivalence}
	\PP(X_k = X_0 | X_0 = U) 
		= \sum_{x \in \Omega} \PP(X_k = x | X_0 = x) \PP(U = x) =  {1 \over |\Omega|} \sum_{x \in \Omega} P^k(x,x).
\end{align}
	The rationale of the ${\tt UCPI}$ algorithm can be summarized as proposition~\ref{prop:power_iteration}. This algorithm may be seen as a low-complexity, sampling-based version of Power Iteration.
\begin{proposition}\label{prop:power_iteration}
	If $U$ is uniformly distributed on $\Omega$, then:
	$$
		\lambda_2 = \lim_{k \to \infty} \Big[|\Omega| \PP(X_k = X_0 | X_0 = U) - 1\Big]^{1 \over k}.
	$$
\end{proposition}
\begin{proof}
	From equation~\eqref{eq:trace_proba_equivalence}, we have $\sum_{x \in \Omega} P^k(x,x) = |\Omega|\PP(X_k = X_0 | X_0 = U)$. Replacing the sum in \ref{eq:approximation_convergence} gives the result.\end{proof} We use the following notation:
$$
	m_k = \PP(X_k = X_0 | X_0 = U) = {1 \over |\Omega|} \sum_{i=1}^{|\Omega|} \lambda_i^k.
$$

\subsection{Bias-variance trade-off}

Given a budget of $n$ samples, in order to estimate $\lambda_2$ from proposition~\ref{prop:power_iteration}, one must estimate $m_k$ for $k$ large. To do so, we draw $I = \lfloor {n \over k} \rfloor$ independent sample paths of length $k$ denoted by $$(X_t^1)_{t=0,\dots,k},\dots,(X_t^I)_{t=0,\dots,k},$$  where each sample path starts in a state chosen uniformly at random. An unbiased estimate for $m_k$ is:
$$
	{\hat m}_k = {1 \over I} \sum_{j=1}^I \indic\{ X_k^j = X_0^j \}.
$$
and we estimate $\lambda_2$ with the plug-in estimator:
$$
	{\hat q}_k = \Big[|\Omega| {\hat m}_k - 1\Big]^{1 \over k}.
$$
It remains to choose the value of $k$ to minimize the estimation error. At first sight, choosing $k$ as large as possible seems to be optimal since $[|\Omega| m_k - 1]^{1 \over k}$ tends to $\lambda_2$. However the variance of the resulting estimate grows rapidly with $k$. 

\rev{Let us evaluate the asymptotic distribution of the resulting estimate. As shown in propostion~\ref{prop:deltamethod} (see Appendix), for any fixed $k$ we have:
	$$
		\sqrt{I}\Big( {\hat q}_k -  \Big[ |\Omega| m_k - 1 \Big]^{1 \over k}\Big) \underset{I \to \infty}{\overset{(d)}{\to}} {\cal N}\Big(0, {m_k(1-m_k) |\Omega|^2 \over k^2} \Big[|\Omega| m_k - 1\Big]^{2(1-k) \over k}\Big)
	$$
  Now in general, $$m_k \underset{k \to \infty}{\sim} {1 + \lambda_2^k \over |\Omega|},$$ so for large $k$ the asymptotic variance of the estimate ${\hat q}_k$ becomes:
$$
	{m_k(1-m_k) |\Omega|^2 \over I k^2} \Big[|\Omega| m_k - 1\Big]^{2(1-k) \over k} \underset{k \to \infty}{\sim}  { |\Omega| - 1 \over k^2 I \lambda_2^{2(k-1)}}.
$$
}
The variance grows exponentially with $k$ so that $k$ must be chosen carefully to balance bias and variance. Choosing the optimal value of $k$ is not straightforward without computing the eigenvalues of $P$ (which is precisely what we are trying to estimate).

We solve the bias-variance trade-off as follows. We compute the above estimates for well chosen range of values for $k=1,\dots,K$, and for each $k$, we replace ${\hat m}_k$ by a confidence upper bound ${\hat u}_k$, and we select $k$ yielding the smallest confidence upper bound on $\lambda_2$. This tends to favor larger values of $k$ , since $k \mapsto \Big[|\Omega| m_k - 1\Big]^{1 \over k}$ is decreasing, while avoiding values such that the variance of ${\hat q}_k$ blows up.

\subsection{Algorithm statement} We now provide a full description of ${\tt UCPI}$ as Algorithm~\ref{alg:ucpi}. The ${\tt UCPI}$ algorithm has three input parameters: $I$ the number of sample paths, $K$ - the length of sample paths and $\delta$ - a confidence parameter. The algorithm performs $K I$ calls to the ${\tt NextState}$ sub-routine, so that for a budget of at most $n$ calls, one may choose $I = \lfloor {n \over K}\rfloor$ and $K \in \{1,\dots,n\}$. The algorithm uses a sub-routine ${\tt CB}$ to compute a confidence upper bound 
$$
{\tt CB}(m,I,\delta)  = \max\{ u \in [m,1] : I \times D\left(m || u \right) \le \ln(1/\delta) \}
$$
with confidence level $\delta$ and
$$
	D(u || v) = u \ln {u \over v} + (1-u) \ln  {1-u \over 1-v},
$$
the Kullback-Leibler divergence between Bernoulli distributions with expectations $u$ and $v$. Sub-routine ${\tt CB}$ is indeed a confidence upper bound from Fact~\ref{fact:ucb} (in the appendix see Fact~\ref{fact:chernoff2}), a consequence of Chernoff's inequality.
\begin{fact}\label{fact:ucb}
	Consider ${\hat m}_k \sim $ Binomial($m_k,I$). Then: $$\PP( {\tt CB}({\hat m}_k,I,\delta) \ge m_k) \ge 1 - \delta.
$$
\end{fact}

\begin{algorithm}[th]
   \caption{{\tt UCPI} (Upper Confidence Power Iteration)}
   \label{alg:ucpi}
\begin{algorithmic}
   \STATE {\bfseries Input:} $P$ transition matrix, $I$ number of paths, $K$ max. length of paths , $\delta$ confidence parameter
   \STATE Initialize ${\hat m}_k \gets 0$ , for $k=1,\dots,K$;
   \FOR{$i=1$ {\bfseries to} $I$}
   \STATE $X_0 \gets {\tt Uniform}(\Omega)$; $X \gets X_0$ ; 
   \FOR{$k=1$ {\bfseries to} $K$} \STATE $X \gets {\tt NextState}(X)$ ; ${\hat m}_k \gets {i-1 \over i}{\hat m}_k + {1 \over i} {\bf 1}\{ X_0 = X\}$;   
   \ENDFOR
	\ENDFOR	
	\FOR{$k=1$ {\bfseries to} $K$}
   \STATE ${\hat u}_k \gets   {\tt CB}({\hat m}_k,{\delta \over 2K})$; ${\hat \ell}_k \gets \min\left(  (|\Omega| {\hat u}_k - 1)^{1 \over k} , 1 \right) $;
	\ENDFOR
	\STATE $\hat{\ell}_\star \gets \min_{k=1,\dots,K} {\hat \ell}_k$;
	\STATE {\bfseries Output:} $\hat{\ell}_\star$ an estimate and confidence upper bound for $\lambda_\star$.
\end{algorithmic}
\end{algorithm} 

\subsection{Parallel implementation} ${\tt UCPI}$ is amenable to parallel implementation in a straightforward manner. Indeed, since ${\tt UCPI}$ draws $I$ independent sample paths of the Markov chain of interest and computes an average over those $I$ sample paths, this task can be done in batches using several processors working in parallel. Methods that involve manipulating the matrix $P$ do not seem to lend themselves to parallel implementation in such a simple manner.

\section{Analysis of UCPI}\label{sec:genericAnalysis}
In this section, we provide a mathematical analysis of the estimation error of ${\tt UCPI}$. We also show how to set its input parameters $I, K$ and $\delta$ to ensure good performance.
\subsection{A generic analysis} We first outline a generic analysis of {\tt UCPI} for any values of its input parameters $K,I$ and $\delta$. We shall then explain how these parameters should be chosen. We define:\rev{
\begin{align*}
	{\cal E}_{k,1} = (|\Omega|m_k  - 1)^{1 \over k} - \lambda_2 \;,\; 
	{\cal E}_{k,2} = {|\Omega| \sqrt{ 32 m_k \ln(K/\delta) } \over \sqrt{I} k (|\Omega|m_k  - 1)^{1 - {1 \over k}} } \; \text{ and } \;
	{\cal E}_k = {\cal E}_{k,1} + {\cal E}_{k,2}.
\end{align*}}
where for any value of the sample path length $k$, ${\cal E}_{k,1}$ represents the estimation bias due to the fact that we consider $k$ finite, and ${\cal E}_{k,2}$ represents the standard deviation due to the fact that we estimate $m_k$ using the random quantity ${\hat m}_k$. The last term ${\cal E}_k$ is therefore the sum of both error terms, which we aim at minimizing. 

Our generic analysis of {\tt UCPI} is stated as Theorem~\ref{th:generic}. The proof is given in the appendix. 
\begin{theorem}\label{th:generic}
	Consider $I ,K,\delta$ such that:
	$$
		{2 \ln(2K/\delta) \over I} \le {1 \over |\Omega|},
	$$
	and denote by $\hat{\ell}_\star$ the output of ${\tt UCPI}$. 
	
	Then we have:
\begin{align*}
	 \PP  \left( \lambda_2 \le {\hat \ell}_\star \le \lambda_2 + \min_{k=1,\dots,K} {\cal E}_k \right) \ge 1 - \delta \; \text{ and } \;
	 \EE |\hat{\ell}_\star - \lambda_2| \le \delta +  \min_{k=1,\dots,K} {\cal E}_k. 
\end{align*}
\end{theorem}
Several facts should be noted. First, the output of {\tt UCPI}, ${\hat \ell}_\star$ is not only an estimate of $\lambda_2$, but it is in fact a confidence upper bound with level $\delta$. Therefore ${\tt UCPI}$ provides a confidence \emph{upper} bound on the relaxation time:
$$
	\PP\Big( {1 \over 1 - {\hat \ell}_\star} \ge t_{r} \Big) \ge 1 - \delta.
$$
This is important since in practice we are interested in \emph{conservative} estimates of the relaxation time, to ensure that the Markov chain has gone through enough transitions to be close to its stationary distribution. The second notable fact is that the estimation error $|\hat{\ell}_\star - \lambda_2|$ is (with probability greater than $1-\delta$) smaller than $\min_{k=1,\dots,K} {\cal E}_k$. Hence everything happens as if ${\tt UCPI}$ chose the best value of $k$, without any additional information. We will elaborate on this in the next sub-sections.

\subsection{Scaling of the estimation error}
We now investigate how the estimation error of ${\tt UCPI}$ scales when $n \to \infty$, and how to set the input parameters of ${\tt UCPI}$ as a function of $n$ to ensure good performance. We define:
$$
	r = {\ln 1/\lambda_2 \over \ln 1/\lambda_3 } \in [0,1],
$$
which depends on the second spectral gap, and it is noted that, when $\lambda_2$ and $\lambda_3$ are close to $1$, $r$ is close to the ratio between the second and the third spectral gap ${1 - \lambda_2 \over 1 - \lambda_3}$.
\begin{proposition}\label{prop:scaling} 
Define $\Delta(K,I,\delta) = { 64 \ln(K/\delta)  \over |\Omega| I}$ and $k^\star(K,I,\delta) = {\ln(1/\Delta(K,I,\delta)) \over 2 \ln(1/\lambda_3)}$. Consider $K,\delta,I$ satisfying both (i) $k^\star(K,I,\delta) \le K$, and (ii) $\lambda_2^{k^\star(K,I,\delta)} \le {1 \over |\Omega|}$. Then we have:
	$$
		\min_{k=1,\dots,K} {\cal E}_k \le {4 |\Omega| \ln(1/\lambda_3) \over \ln({ |\Omega| I \over 64 \ln(K/\delta)})}  \left( 128 \ln(K/\delta)  \over |\Omega| I\right)^{{1 - r \over 2}}.
$$
\end{proposition}
\rev{Conditions (i) and (ii) of proposition~\ref{prop:scaling} stem from the fact that $K$ should be chosen large enough to allow selecting a value of $k \in \{1,\dots,K\}$ such that the bias ${\cal E}_{k,1}$ is small enough, as the bias decreases with $k$. Further, for conditions (i) and (ii) to hold, one must select $I$ and $K$ to verify $I \ge {|\Omega|^{{r+1 \over r-1}}}$ and $K \ge \ln |\Omega|$, otherwise we cannot guarantee that the error vanishes.}
\subsection{Setting the input parameters of UCPI}
We finally show how the input parameters of {\tt UCPI} $K$ and $\delta$ may be selected to ensure asymptotically optimal performance when compared to an oracle which could determine their optimal values. We show in particular that the error of {\tt UCPI} with properly chosen parameters scales as ${\cal O}\left({|\Omega|^{r + 1 \over 2} \ln(1/\lambda_3) \over (\ln n)^{{3r-1 \over 2}} n^{{1 - r \over 2}}} \right)$ when $n \to \infty$.
\begin{proposition}\label{prop:convergence_rate}
	Consider $\hat{\ell}_\star$ the output of {\tt UCPI} with parameters $\delta = {1 \over \sqrt{n}}$ and $K = (\ln n)^2$ and $I = {n \over K} = {n \over (\ln n)^2}$. There exists $C > 0$ a universal constant and $n_0(|\Omega|,\lambda_2,\lambda_3)$ such that:
	\begin{align*}
		\EE |\hat{\ell}_\star - \lambda_2| &\le C {|\Omega|^{r + 1 \over 2} \ln(1/\lambda_3) \over (\ln n)^{{3r-1 \over 2}} n^{{1 - r \over 2}}}. \;\;,\;\; \forall n \ge  n_0.
	\end{align*}
\end{proposition}

\subsection{Computational complexity}\label{subsec:complexity}

For simplicity, we assume that sub-routine $\tt CB$ can be implemented in time ${\cal O}(1)$. In all generality, implementing $\tt CB$ requires to find the zero of function $u \mapsto I D\left(m || u \right) - \ln(1/\delta)$ on $[m,1)$ which is stricly increasing and convex, and one needs to use an iterative method. However, if Newton's method is used, the value of $\tt CB$ can be computed up to accuracy $\varepsilon$ in time ${\cal O}(\ln \ln 1/\varepsilon)$ so that, for say $10$ iterations, the machine precision is reached on any modern computer.

Since $K \le I K = n$, running ${\tt UCPI}$ requires time ${\cal O}(n)$ and memory space ${\cal O}(K)$. 
The time complexity is minimal since any algorithm making use of $n$ samples should inspect each sample at least one time. As our theoretical analysis demonstrates, one should set $K = {\cal O}((\ln n)^2)$ to ensure good performance. 
Hence  ${\tt UCPI}$ requires ${\cal O}(n)$ time and ${\cal O}((\ln n)^2)$ space. This is in stark contrast with other methods storing $P$ in memory.  If $n$ and $\Omega$ are of comparable order, this represents a dramatic memory reduction from ${\cal O}(n^2)$ to ${\cal O}((\ln n)^2)$.

Assume that $r < 1$ is fixed. To get an estimation error of $\epsilon$, {\tt UCPI} requires at most $n_1 = {\cal O}( \epsilon^{-{2 \over 1-r}} |\Omega|^{ r + 1 \over r-1})$ samples, time ${\cal O}(n_1)$ and memory ${\cal O}((\ln n_1) ^2 )$. By comparison, the method of \cite{Hsu2015} requires at least $n_2 = {\cal O}((\epsilon \pi_m (1 - \lambda_\star))^{-2})$ samples, with time ${\cal O}(n_2 + |\Omega|^3)$ and memory ${\cal O}(|\Omega|^2)$, where $\pi_m = \min_{x \in \Omega} \pi(x)$. It is noted that $(\pi_m)^{-1} > |\Omega|$, and that $\pi_m$ may be arbitrarily small, so that $n_2$ is not necessarily smaller than $n_1$. {\tt UCPI} requires memory space ${\cal O}((\ln { |\Omega|\over \epsilon})^2)$ as opposed to ${\cal O}(|\Omega|^2)$. 
This is a large improvement since in practice $\ln {|\Omega|\over \epsilon}$ is much smaller than $\Omega$ (even for $\epsilon = 10^{-7}$ and $|\Omega| = 10^3$, the memory is reduced by $4$ orders of magnitude). {\tt UCPI} is easily amenable to parallel implementation, so trading (a bit of) time against (a lot of) memory seems to be be a good perspective.

\rev{Finally, as shown by \cite{Hsu2015}, in order to reach any fixed accuracy, $n$ should satisfy $n \ge {\cal O}({1 \over \pi_m}) \ge {\cal O}(|\Omega|)$. So it is not possible to design an algorithm whose sample complexity is sub-linear in the state space size.} 
\section{Extensions}\label{sec:extensions}
In this section we show how some of the assumptions on the Markov chain and/or the computational model may be broadened while ensuring that our results still apply with minimal modifications.

\subsection{Non-lazy chains} Assume that $P$ is reversible, irreducible but not lazy so that its eigenvalues are not necessarily positive. Denote by $$\zeta_{1} > \zeta_{2} \ge \zeta_3 \ge ... \ge \zeta_{|\Omega|},$$ the sorted eigenvalues of $P^2$ which are positive. Then $\zeta_{2} = (\lambda_\star)^2$ so that the spectral gap of $P$ is $1 - \sqrt{\zeta_2}$. Hence the spectral gap of $P$ can be retrieved by applying ${\tt UCPI}$ to $P^2$, which is done by calling  the ${\tt NextState}$ function twice for each transition. Also, an error $\epsilon$ when estimating $\zeta_2$ results in an error of ${\cal O}({\epsilon \over \lambda_2})$ when estimating $\lambda_2$, which is ${\cal O}(\epsilon)$ when $\zeta_2$ is close to $1$, which is the hardest case.
	
\subsection{General initial distributions} Now assume that it is not feasible to sample from the uniform distribution over $\Omega$, so that we do not have access to the ${\tt Uniform}$ subroutine. For instance, when $\Omega$ is a large discrete set in high dimension, the only efficient method to sample uniformly over $\Omega$ is often to use MCMC methods such as hit-and-run or other random walks \cite{LovaszMCMC,LovaszMCMC2,KannanMCMC}. After enough MCMC steps, one obtains a distribution that is close to uniform, but not exactly uniform. Formally, consider the case where one has access to a sub-routine ${\tt Sample}$ whose output is a random variable with distribution $\mu$:
	$$
		\PP({\tt Sample}() = x) = \mu(x) \,,\, x \in \Omega. 
	$$
	with $\min_{x \in \Omega} \mu(x) > 0$. Express the trace of $P$ using a change of measure:
	\begin{align*}
		\sum_{x\in\Omega} P^k(x,x) = \sum_{x\in\Omega} {1 \over \mu(x)} \mu(x) P^k(x,x)
						= \EE\Big( {\indic\{ X_k = X_0 \} \over \mu(X_0)}   | X_0 \sim \mu \Big).
	\end{align*}
	So proposition~\ref{prop:power_iteration} can be extended to this case, and ${\tt UCPI}$ can be modified by drawing $I$ independent sample paths of length $k$ with initial distribution $\mu$. 
	\begin{proposition}\label{prop:power_iteration2}
	We have $$
		\lambda_2 = \lim_{k \to \infty} \Big[\EE_{U \sim \mu} \Big( {\indic\{ X_k = X_0 \} \over \mu(U)}  | X_0 = U \Big) - 1\Big]^{1 \over k}.
	$$
\end{proposition}
	Furthermore, if $\min_{x \in \Omega} \mu(x)$  and ${1 \over |\Omega|}$ are of the same order, our analysis still applies, and the same bounds on the convergence rate hold. The proof of proposition~\ref{prop:convergence_rate_mu} is straightforward from the arguments given in section~\ref{sec:genericAnalysis} and is omitted. 
	\begin{proposition}\label{prop:convergence_rate_mu}
	Consider $\hat{\ell}_\star$ the output of {\tt UCPI} with parameters $\delta = {1 \over \sqrt{n}}$ and $K = (\ln n)^2$ and $I = {n \over K} = {n \over (\ln n)^2}$ in the extended setting described above. Assume that $\min_{x \in \Omega} \mu(x) \ge {C' \over |\Omega|}$ for $C'$ a universal constant. There exists $C'' > 0$ a universal constant and $n_0'(|\Omega|,\lambda_2,\lambda_3)$ such that 
	\begin{align*}
		\EE |\hat{\ell}_\star - \lambda_2| &\le C'' {|\Omega|^{r + 1 \over 2} \ln(1/\lambda_3) \over (\ln n)^{{3r-1 \over 2}} n^{{1 - r \over 2}}} \;\;,\;\; \forall n \ge  n_0'.
	\end{align*}
\end{proposition}

\subsection{Reduction from USP to RTF} 

\subsubsection{Reduction} Finally we show that the USP model can be efficiently reduced to the RTF model, and therefore ${\tt UCPI}$ applies there as well. In the USP model we are provided with a unique sample path $(X_t)_{t \ge 0}$ with transition matrix $P$ and arbitrary starting distribution. To apply ${\tt UCPI}$, we need to generate sample paths $(Y_t)_{t=0,\dots,k}$ of length $k$ starting from a uniform distribution i.e. $(Y_t)_{t=0,\dots,k}$ is a Markov chain with transition matrix $P$ and $Y_0$ is uniformly distributed. Consider $U$ uniformly distributed on $\Omega$, and define $\tau$ the smallest $t$ such that $X_t = U$. Then from the strong Markov property, $Y_t = X_{\tau + t}$, for $t=0,...,k$ is a Markov chain with transition matrix $P$ and uniform initial distribution. Furthermore, applying this procedure repeatedly yields i.i.d. sample paths, which is sufficient to apply ${\tt UCPI}$.	
\begin{proposition}\label{prop:resampling}
	Consider $(X_t)_{t \ge 0}$ a sample path of a Markov Chain with transition matrix $P$ and arbitrary starting distribution.	Consider $(Z_{t})_{t=0,\dots,k}$ a sample path of length $k$ of a Markov chain with transition matrix $P$ and uniform initial distribution.	Consider $(U^i)_{i \ge 1}$ i.i.d. uniform random variables on $\Omega$ independent of  $(X_t)_{t \ge 0}$. Define the random stopping times:
	$
		\tau^i = \min\{t > \tau^{i-1} + k : X_t = U^i \},
	$
	with the convention that $\tau^0 = -k$. Finally define 
	$
	Y_{t'}^i = X_{\tau^i + t'} \,\text{ for } \, i \ge 0 \,\text{ and }\, t'=0,\dots,k.
	$
	Then $(Y_{t}^i)_{t=0,\dots,k,i \ge 1}$ are i.i.d copies $(Z_{t})_{t=0,\dots,k}$. 	
\end{proposition}
The proof of proposition~\ref{prop:resampling} is a simple consequence of the strong Markov property for Markov chains and is omitted. Given a unique sample path of length $n$, the procedure described in proposition~\ref{prop:resampling} allows to draw i.i.d. sample paths of length $k$ with uniform initial distribution in time ${\cal O}(n)$. It is also noted that this procedure involves inspecting each element of the sequence $(X_t)_{t \ge 0}$ exactly once i.e it is a \emph{streaming} algorithm and is adapted to the case where $(X_t)_{t \ge 0}$ is provided as a data stream.  We have shown that ${\tt UCPI}$ applies to both the RTF and the USP model as well, and its computational complexity is the same in both models.

	\subsubsection{Asymptotic efficiency} In fact, perhaps surprisingly, the reduction from USP to RTF comes at virtually no cost when the number of samples $n$ is large. Consider a budget of $n$ calls to {\tt NextState}, and $K = (\ln n)^2$ the length of a sample path that we must generate. In the RTF model, we can generate  $I^{\text{RTF}}_n$ sample paths with 
	$$
	I^{\text{RTF}}_n = {n \over (\log n)^2}.
	$$
	In the USP model, using the reduction, the expected number of calls in order to generate one sample path of length $K$ is $K + \xi_{max} = (\ln n)^2 + \xi_{max}$ where $\xi(x,y)$ is the expected value of the first hitting time of state $y$ starting from $x$ and 
	$$
	\xi_{max} = \max_{(x,y) \in \Omega^2} \xi(x,y) < \infty.
	$$	
	 It is noted that $\xi_{max}$ does not depend on $n$. Therefore $n$ calls to {\tt NextState} enable us to generate $I^{\text{USP}}_n$ sample paths
	 $$
	 I^{\text{USP}}_n = {\cal O}\left({n \over \xi_{max} + (\ln n)^2}\right).
	$$	 
  Asymptotically, as $n \to \infty$, we have that ${I^{\text{USP}}_n \over I^{\text{RTF}}_n} \to 1$, so that the number of sample paths generated in both models is equivalent.
\subsection{Dependency of the results on the spectrum} We have stated our results in the case where $\lambda_2 > \lambda_3$ to simplify the exposition, i.e. $r < 1$. In fact, a similar analysis can be carried out without this assumption, where $r$ is replaced by a function which quantifies the decay of the eigenvalues $\lambda_3,...,\lambda_{|\Omega|}$. This quantifies the convergence speed of $( \sum_{i \ge 2} \lambda_i^k )^{1/k}$ to $\lambda_2$ when $k$ grows. It is also noted that, even when $r = 1$, our analysis shows that the error of {\tt UCPI} vanishes, albeit at a slower rate of ${\cal O}\left( {|\Omega| \over \ln n}\right)$. 
\section{Numerical Experiments}\label{sec:numerical_experiments}
We now evaluate the numerical performance of {\tt UCPI} using numerical experiments, and compare it to that of the {\tt UCI} algorithm proposed in \cite{Hsu2015}. 
\subsection{Instances considered}
We consider two families of Markov Chains for our experiments. 

\subsubsection{Case 1} The first case is a biased random walk on a line, so that the state space is $\Omega = \{1,\dots,|\Omega|\}$, and the transition matrix is:
$$
	P(x,y) = \begin{cases} 
			{1 \over 2} &\text{ if } y = x,\\
		    {1-p \over 2} &\text{ if } y = x+1, \\
		    {p \over 2} &\text{ if } y = x-1, \\
			0 &\text{ otherwise.} 
			\end{cases}
$$
Here $p \in (0,1)$ denotes the bias, so that for $p = 0.5$ this is a classical random walk over a line, and for $p > 0.5$ the random walk has a negative drift. This chain is lazy and reversible since it verifies detailed balance, with stationary distribution
$$
\pi(x) = \begin{cases} {({1 \over p} - 1)^{x} (2 - {1 \over p}) \over 1 - ({1 \over p} - 1)^{|\Omega|}} & \text{ if } p \ne {1 \over 2} \\
{1 \over |\Omega|}  & \text{otherwise}.
\end{cases}
$$

\subsubsection{Case 2} The second case is a lazy random walk on a $d$-regular graph $G = (V,E)$. We consider the Markov chain on state space $\Omega = V$, with transition probabilities
$$
	P(x,y) = \begin{cases} 
			{1 \over 2} &\text{ if } y = x,\\
		    {1 \over 2d} &\text{ if } (x,y) \in E,\\
		    0 & \text{otherwise.}
		    \end{cases}
$$
Once again, this chain is lazy and reversible since it verifies detailed balance with stationary distribution 
$$
\pi(x) = {1 \over |\Omega|}.
$$ 
\subsection{Estimation error}

For both algorithms, given $n$ samples i.e. $n$ calls to {\tt NextState}, we consider the confidence upper bound on $\lambda_{\star}$ outputted by said algorithm denoted by ${\hat \ell_\star}$. For both algorithms, the level of the confidence interval is chosen as $\delta = {1 \over \sqrt{n}}$, so that:
$$
	\PP( {\hat \ell_\star} \le \lambda_{\star} ) \le \delta = {1 \over \sqrt{n}}.
$$
This confidence upper bound gives a corresponding confidence upper bound on $t_{r}$, which we denote by 
$${\hat t}_r = {1 \over 1 - {\hat \ell_\star}}.$$ 
In tables~\ref{table:ex1} and~\ref{table:ex2} we depict the behavior of both algorithms for various sample sizes $n$. Several remarks are in order. First both algorithms have a threshold value of $n$ before which their output is not informative i.e. ${\hat \ell_\star} = 1$, and ${\hat t}_r = \infty$. This is important because, before this threshold, one has no information about the mixing properties of the Markov chain at hand, so that one does not know for how long to sample the chain to be close to the stationary distribution, which is the very reason why one would use such algorithms in practice.

Even for a very large sample size of $n = 10^8$, the output of {\tt UCI} is not informative (see next subsection for a further investigation of this issue). On the other hand, for both cases, the output of {\tt UCPI} is informative when $n \ge 10^6$, and its confidence upper bound ${\hat \ell_\star}$ decreases towards the quantity of interest $\lambda_\star$. Therefore, not only does {\tt UCPI} have low complexity, it also performs best statistically in both considered cases.

\begin{table}[!htbp]
\centering
\begin{tabular}{|c|c|c|c|c|}
   \hline
   \multicolumn{5}{|c|}{$|\Omega| = 20$, $p=0.5$, $\lambda^\star = 0.99$, $t_{r} = 100$} \\
   \hline
   $n$ & {\tt UCPI} ${\hat \ell_\star}$  & {\tt UCPI} ${\hat t}_r$  &   {\tt UCI} ${\hat \ell_\star}$  & {\tt UCI} ${\hat t}_r$ \\
   \hline
   $10^4$ & $1$ & $\infty$ & $1$ & $\infty$ \\
   $10^5$ & $1$ & $\infty$ & $1$ & $\infty$ \\      
   $10^6$ & $0.996$ & $272$ & $1$ & $\infty$ \\      
   $10^7$ & $0.994$ & $176$ & $1$ & $\infty$ \\      
   $10^8$ & $0.993$ & $154$ & $1$ & $\infty$  \\               
   \hline
   \multicolumn{5}{|c|}{$|\Omega| = 20$, $p=0.7$, $\lambda^\star = 0.95$, $t_{r} = 20$} \\
   \hline
   $n$ & {\tt UCPI} ${\hat \ell_\star}$  & {\tt UCPI} ${\hat t}_r$  & {\tt UCI} ${\hat \ell_\star}$  & {\tt UCI} ${\hat t}_r$ \\
   \hline
   $10^4$ & $1$ & $\infty$ & $1$ & $\infty$ \\
   $10^5$ & $0.993$ & $151$ & $1$ & $\infty$ \\      
   $10^6$ & $0.977$ & $45$ & $1$ & $\infty$ \\      
   $10^7$ & $0.969$ & $32$ & $1$ & $\infty$ \\      
   $10^8$ & $0.963$ & $27$ & $1$ & $\infty$ \\               
   \hline
   \multicolumn{5}{|c|}{$|\Omega| = 20$, $p=0.9$, $\lambda^\star = 0.8$, $t_{r} = 5$} \\
   \hline
   $n$ & {\tt UCPI} ${\hat \ell_\star}$  & {\tt UCPI} ${\hat t}_r$  & {\tt UCI} ${\hat \ell_\star}$  & {\tt UCI} ${\hat t}_r$ \\
   \hline
   $10^4$ & $1$ & $\infty$ & $1$ & $\infty$ \\
   $10^5$ & $0.985$ & $69$ & $1$ & $\infty$ \\      
   $10^6$ & $0.939$ & $16$ & $1$ & $\infty$ \\      
   $10^7$ & $0.899$ & $10$ & $1$ & $\infty$ \\      
   $10^8$ & $0.875$ & $8$ & $1$ & $\infty$ \\               
   \hline
\end{tabular}
   \caption{\label{table:ex1} Biased random walk on a line.}
\end{table}

\begin{table}[!htbp]
\centering
\begin{tabular}{|c|c|c|c|c|}
   \hline
   \multicolumn{5}{|c|}{$|\Omega| = 100$, $d = 5$, $\lambda^\star = 0.88$, $t_{r} = 8.3$} \\
   \hline
   $n$ & {\tt UCPI} ${\hat \ell_\star}$  & {\tt UCPI} ${\hat t}_r$  &   {\tt UCI} ${\hat \ell_\star}$  & {\tt UCI} ${\hat t}_r$ \\
   \hline
   $10^4$ & $1$ & $\infty$ & $1$ & $\infty$ \\
   $10^5$ & $1$ & $\infty$ & $1$ & $\infty$ \\
   $10^6$ & $0.979$ & $49$ & $1$ & $\infty$ \\
   $10^7$ & $0.958$ & $24$ & $1$ & $\infty$ \\      
   $10^8$ & $0.934$ & $15$ & $1$ & $\infty$ \\         
   \hline
   \multicolumn{5}{|c|}{$|\Omega| = 100$, $d=10$, $\lambda^\star = 0.78$, $t_{r} = 4.5$} \\
   \hline
   $n$ & {\tt UCPI} ${\hat \ell_\star}$  & {\tt UCPI} ${\hat t}_r$  &   {\tt UCI} ${\hat \ell_\star}$  & {\tt UCI} ${\hat t}_r$ \\
   \hline
   $10^4$ & $1$ & $\infty$ & $1$ & $\infty$ \\
   $10^5$ & $0.962$ & $26$ & $1$ & $\infty$ \\
   $10^6$ & $0.980$ & $47$ & $1$ & $\infty$ \\
   $10^7$ & $0.941$ & $17$ & $1$ & $\infty$ \\      
   $10^8$ & $0.902$ & $10$ & $1$ & $\infty$ \\     
   \hline
\end{tabular}
   \caption{\label{table:ex2} Random walk on a $d$-regular graph.}
\end{table}

\subsection{Informativeness of confidence intervals}

We have seen that {\tt UCI} does not provide an informative confidence interval even given a large number of samples (say $n = 10^8$) for relatively benign Markov chains, due to the fact that it needs sample paths on which \emph{all} states are seen a very large number of times. Namely we must have:
$$
\min_{x \in \Omega} \left(\sum_{t=1}^n \indic\{ X_t = x\} \right) \gg 1
$$
before an informative confidence upper bound can be outputted. In table \ref{table:ex3}, for $n = 10^6$, we evaluate experimentally $\PP({\hat \ell_\star} < 1)$, the probability that {\tt UCI} and {\tt UCPI} output an informative confidence upper bound (i.e. ${\hat \ell_\star} < 1$) by simulating $10^3$ independent experiments. As noted before, the output of {\tt UCPI} is informative with high probability, while that of {\tt UCI} is uninformative with high probability.

	\rev{In light of these numerical results we observe that while {\tt UCI} provably attains the lower bound derived in \cite{Hsu2015}, and is hence minimax optimal, it does not perform so well numerically, while {\tt UCPI} does not attain the optimal rate (its error rate depends on $r$ defined above) given by the lower bound of \cite{Hsu2015}, but does very well numerically. Therefore, it might be interesting to look at problem-specific optimality for the design of future algorithms.}
	
\begin{table}[!htbp]
\centering
\begin{tabular}{|c|c|c|}
   \hline
   $n = 10^6$  & {\tt UCPI} $\PP({\hat \ell_\star} < 1)$ & {\tt UCI} $\PP({\hat \ell_\star} < 1)$ \\
   \hline
   Case 1 ($p = 0.5$) & $1$ & $0$ \\
   Case 1 ($p = 0.7$) & $1$ & $0$\\
   Case 1 ($p = 0.9$) & $1$ & $0$\\
   \hline
   Case 2 ($d = 5$) & $1$ & $0$\\
   Case 2 ($d = 10$) & $1$ & $0$\\
   \hline
\end{tabular}
   \caption{\label{table:ex3} Informativeness of confidence upper bounds.}
\end{table}
\section{Conclusion}\label{sec:conclusion}
We have proposed and analyzed ${\tt UCPI}$ a computationally efficient algorithm for estimating the spectral gap of a Markov chain using time ${\cal O}(n)$ and space ${\cal O}( (\ln n)^2)$, unlike known estimation procedures which require (at least) ${\cal O}( |\Omega|)$ space and do not apply to large state spaces typically found in applications such as MCMC. We believe that our results open the following challenging question: ``What is the optimal trade-off between computation and statistical accuracy to estimate the mixing properties of a Markov chain ?''. \rev{Furthermore while we have focused on the relaxation time, tighter mixing bounds could be obtained by estimating the mixing time instead, which seems much more challenging, especially in the black-box setting.}
\bibliographystyle{plain}
\bibliography{./biblio}
\appendix
\section{Application of the delta method}
\rev{

\begin{proposition}\label{prop:deltamethod}
	Consider $k$ fixed, $(M_k^j)_j$ i.i.d. Bernoulli with expectation $m_k = {1 \over |\Omega|} \sum_{i=1}^{|\Omega|} \lambda_i^k$. 
	
	Let ${\hat m}_k =  {1 \over I} \sum_{j=1}^I M_k^j$ and $	{\hat q}_k = \Big[|\Omega| {\hat m}_k - 1\Big]^{1 \over k}$. Then:
	$$
		\sqrt{I}\Big( {\hat q}_k -  \Big[ |\Omega| m_k - 1 \Big]^{1 \over k}\Big) \underset{I \to \infty}{\overset{(d)}{\to}} {\cal N}\Big(0, {m_k(1-m_k) |\Omega|^2 \over k^2} \Big[|\Omega| m_k - 1\Big]^{2(1-k) \over k}\Big)
	$$
\end{proposition}
{\bf Proof :} Since ${\hat m}_k$ follows a binomial distribution with mean $m_k$ and variance ${m_k(1-m_k) \over I}$, the central limit theorem gives:
$$
	\sqrt{I}({\hat m}_k  - m_k) \underset{I \to \infty}{\overset{(d)}{\to}} {\cal N}(0,m_k(1-m_k)),
$$ 
Define $f_k(x) = \Big[ |\Omega| x - 1 \Big]^{1 \over k}$. Note that $|\Omega| m_k - 1 = \sum_{i=2}^{|\Omega|} \lambda_i^k \ge \lambda_2^k > 0$, therefore $x \mapsto f_k(x)$ is continuously differentiable at $x = m_k$. The delta-method (\cite{deltamethod}) yields:
$$
			\sqrt{I}( f_k({\hat m}_k) - f_k(m_k)) \underset{I \to \infty}{\overset{(d)}{\to}} {\cal N}\Big(0, m_k(1-m_k) \Big[f_k'(m_k)\Big]^2\Big).
$$
Replacing $f_k$ by its expression yields the announced result, as $f_k({\hat m}_k) = {\hat q}_k$, $f_k(m_k) = \Big[|\Omega| m_k - 1\Big]^{1 \over k}$ and $f_k'(m_k) = {|\Omega| \over k}\Big[ |\Omega| x - 1 \Big]^{1 - k \over k}$.
}
\section{Preliminary technical lemmas}

We first recall the Chernoff bounds~\cite{Hoeffding}.
\begin{fact}\label{fact:chernoff}(Chernoff bound)
	Consider $(U_1,...,U_I)$ i.i.d. Bernoulli random variables with expectation $u$, and $\bar{U} = {1 \over I} \sum_{i=1}^I U_i$. Then for all $\epsilon > 0$ we have:
	$$
		\PP( \bar{U} \ge u + \epsilon ) \le \exp(-I D( u + \epsilon || u) ),
	$$
	and:
	$$
		\PP( \bar{U} \le u - \epsilon) \le \exp(-I D( u - \epsilon ||u) ).
	$$
\end{fact}

We then state a lower bound on the Kullback-Leibler divergence which is instrumental to our analysis.\\
\begin{fact}\label{fact:kllb}
	For all $u,v$ we have:
	$$
		D(u || v) \ge {(u-v)^2 \over 2 \max\{u,v\}}.
	$$
\end{fact}
{\bf Proof :} Recall the definition of the Kullback-Leibler divergence:
	$$
	    D(u || v) = u \ln\left({u \over v}\right) + (1-u) \ln\left({1-u \over 1-v}\right) 
	$$	
	Differentiating the above equation we obtain:
	$$
		{d \over du} D(u,v) = \ln\left({u(1-v) \over v(1-u)}\right), 
	$$
	and:
	$$
		{d^2 \over du^2} D(u || v) = {1 \over u(1-u)},
	$$
	Consider $v$ fixed. Apply the Taylor-Lagrange equality to function $u \mapsto D(u || v)$, so that there exists $u' \in [\min(u,v),\max(u,v)]$ with:
	\begin{align*}
		D(u || v) &= D(u || v) + (u-v)\left({d \over du} D(u || v) \Big|_{u=v}\right)  
		+  {(u-v)^2 \over 2} \left({d^2 \over du^2} D(u || v) \Big|_{u=u'}\right) \\
			   &= 0 + 0  +  {(u-v)^2 \over 2 u'(1-u')} 
			   \ge {(u-v)^2 \over 2 \max\{u,v\}}.
	\end{align*}
	where we used the fact that $u'(1-u') \le u' \le \max(u,v)$ since $u' \in [\min(u,v),\max(u,v)]$.
\qed

From the above inequality we deduce a simpler (but weaker) version of the Chernoff bound.
\begin{fact}\label{fact:chernoff1}
Consider $U_1,...,U_I$ i.i.d. Bernoulli random variables with expectation $u$, and $\bar{U} = {1 \over I} \sum_{i=1}^I U_i$. Consider $\delta \in (0,1)$ such that $2 \ln(1/\delta) \le u I$. Then: 
	$$
		\PP( \bar{U} \ge u + \sqrt{2 u \ln(1/\delta) / I} ) \le \delta,
	$$
\end{fact}
{\bf Proof:} From Chernoff's inequality with $\epsilon =  \sqrt{2 u \ln(1/\delta) / I}$ we have:
	$$
		\PP( \bar{U} \ge u + \sqrt{2 u \ln(1/\delta) / I} ) \le e^{- I D(u + \sqrt{2 u \ln(1/\delta) / I}  || u)},
	$$
Applying fact~\ref{fact:kllb}:
\begin{align*}
 	D(u + \sqrt{2 u \ln(1/\delta) / I}  || u) &\ge {2 u \ln(1/\delta)  \over I \max\{ u + \sqrt{2 u \ln(1/\delta) / I} , u \} } 
 	 	\ge {\ln(1/\delta) \over I},
\end{align*}
as our assumption $2 \ln(1/\delta) \le u I$ implies $ u + \sqrt{2 u \ln(1/\delta) / I} \le 2 u$. Replacing yields the announced inequality. \qed

We further deduce an upper bound on the upper confidence bounds used in our analysis.
\begin{fact}\label{fact:cb}
For $u \in [0,1]$ and $\Delta > 0$ define:
$$
	V(u,\Delta) = \max\{ v \in [u,1] :  D( u || v) \le \Delta\}.
$$
Then for all $\Delta$ such that  $\Delta \le u$ we have $V(u,\Delta) \le u + \sqrt{9 u \Delta}$.
\end{fact}
{\bf Proof:} Using the previous fact:
$$
	D( u || u + \sqrt{9 u \Delta}) \ge  {9 u \Delta \over 2(u + \sqrt{9 u \Delta}) } \ge  {9  \over 8} \Delta \ge \Delta.
$$
using $u + \sqrt{9 u \Delta} \le 4 u$ since $\Delta \le u$. Combining the above with the definition of $V$ proves that $V(u,\Delta) \le  u + \sqrt{9 u \Delta}$ as announced. \qed

Another notable fact is that $V$ is an upper confidence bound for $u$.
\begin{fact}\label{fact:chernoff2}
Consider $U_1,...,U_I$ i.i.d. Bernoulli random variables with expectation $u$, and $\bar{U} = {1 \over I} \sum_{i=1}^I U_i$. Then for all $\delta \in (0,1)$:
	$$
		\PP( V\left(\bar{U},\ln( 1/ \delta) / I \right) \ge u) \ge 1 - \delta.    
	$$
\end{fact}

\section{Proof of Theorem~1}
For $u \in [0,1]$ and $\Delta > 0$ define:
$$
	V(u,\Delta) = \max\{ v \in [u,1] :  D(u || v) \le \Delta\}.
$$
The first step is to control the random fluctuations of ${\hat u}_k$ around $u_k$. Our first intermediate result is:
$$
	\PP( m_k \le {\hat u}_k \le m_k + \sqrt{ 32 m_k \ln(K/\delta) / I} \,,\,  k\in[K]) \ge 1 - \delta. 
$$
	We first use the fact that ${\hat u}_k$ is a high-probability confidence upper bound for $m_k$. By definition ${\hat u}_k = V({\hat m}_k, \ln(2K/\delta)/I)$ and fact~\ref{fact:chernoff2} gives:
$$
	\PP(  {\hat u}_k \le m_k ) \le {\delta \over 2K}.
$$
	To control the upper fluctuations of ${\hat u}_k$ we first control that of ${\hat m}_k$. It is noted that $$2 \ln(2K/\delta) \le { I \over |\Omega|}  \le m_k I.$$ Hence applying fact~\ref{fact:chernoff1} guarantees that:
	$$
		\PP( {\hat m}_k \ge m_k + \sqrt{ 2 m_k \ln(2K/\delta) / I}) \le {\delta \over 2K}.
	$$
	For $k=1,\dots,K$ we define:
	$$
		a_k = m_k + \sqrt{2 m_k \ln(2K/\delta) / I}.
	$$
	It is noted that, since $2 \ln(2K/\delta) / I \le {1 \over |\Omega|} \le m_k$, replacing we have  $a_k \le 2 m_k$. Now assume that that the event:
$$
	\left\{ {\hat m}_k \le m_k + \sqrt{ 2 m_k \ln(2K/\delta) / I} \right\},
$$
occurs. Then ${\hat m}_k \le a_k$. Since $u \mapsto D( u || v)$ is increasing whenever $u < v$, $u \mapsto V(u,\Delta)$ is increasing as well and we have:
\begin{align*}
	{\hat u}_k &=  V({\hat m}_k, \ln(2K/\delta)/I) 
	\overset{(a)}{\le} V(a_k, \ln(2K/\delta)/I) \\
	&\overset{(b)}{\le} a_k + \sqrt{ 9 a_k \ln(2K/\delta) / I} 
	\overset{(c)}{\le} m_k + \sqrt{32 m_k \ln(2K/\delta) / I}.	
\end{align*}
  where (a) follows from the monotonicity of $u \mapsto V(u,\Delta)$, (b) follows from fact~\ref{fact:cb} and the fact that $\ln(2K/\delta)/I \le m_k \le a_k$, and (c) holds since $a_k \le 2 m_k$ and $(\sqrt{18} + \sqrt{2})^2 = 32$. We have proven that:
$$
	\PP( {\hat u}_k \ge m_k + \sqrt{32 m_k \ln(2K/\delta) / I}) \le \PP( {\hat m}_k \ge a_k) \le {\delta \over 2K}.
$$
Putting everything together using a union bound:
$$
	\PP( m_k \le {\hat u}_k \le m_k + \sqrt{ 32 m_k \ln(K/\delta) / I} \,,\,\forall k \in [K]) \ge 1 - \delta.
$$
Now define the event that all indexes ${\hat u}_k$ do not deviate excessively from $m_k$, that is 
$$
{\cal A} = \cup_{k=1}^K \{m_k \le {\hat u}_k \le m_k + \sqrt{ 32 m_k \ln(K/\delta) / I} \}.
$$
For $k=1,\dots, K$ and $m \in \mathbb{R}$ define the mappings:
$$
f_k(m) = \min\left(( \max\{|\Omega| m - 1,0\})^{1 \over k} , 1 \right).$$
If ${\cal A}$ occurs then, for $k=1,\dots,K$ we have:
$$
	  f_k(m_k) \le {\hat \ell}_k \le f_k(m_k + \sqrt{ C m_k \ln(K/\delta) / I} ),
$$
since ${\hat \ell}_k = f_k({\hat u}_k)$ and $m\mapsto f_k(m)$ is monotonically increasing on $\mathbb{R}$. Since $\lambda_1,\dots,\lambda_{|\Omega|}$ are positive and $\lambda_1=1$, we have:
$$
m_k = {1 \over |\Omega|} \sum_{i=1}^{|\Omega|} \lambda_i^k \ge {1 + \lambda_2^k \over |\Omega|},
$$
and, using once again the monotonicity of $m\mapsto f_k(m)$:
$$
f_k(m_k) \ge f_k\left({1 + \lambda_2^k \over |\Omega|}\right) = \lambda_2.
$$
 Hence, for $k=1,\dots,K$ we have:
$$
	 \lambda_2 \le {\hat \ell}_k \le f_k(m_k + \sqrt{ 32 m_k \ln(K/\delta) / I} ),
$$
Minimizing the above with respect to $k$ we obtain:
$$
	\lambda_2 \le {\hat \ell}_\star \le \min_{k=1,\dots,K}f_k(m_k + \sqrt{ 32 m_k \ln(K/\delta) / I} ).
$$

The final step of our analysis involves upper bounding $f_k$ to simplify the above expression. For $x \ge {1 \over |\Omega|}$:
\begin{align*}
	f_k(x + \delta) = \min\left(  (|\Omega|(x + \delta)- 1)^{1 \over k} , 1 \right) 
			\le   (|\Omega|(x + \delta) - 1)^{1 \over k} 
			\le   (|\Omega|x  - 1)^{1 \over k} + {|\Omega| \delta \over k  (|\Omega|x  - 1)^{1 - {1 \over k}}}.
\end{align*}
where we used the fact that $x \mapsto x^{1 \over k}$ is concave, hence:
$$
(x + \delta)^{1 \over k} - x^{1 \over k} =  {1 \over k} \int_{x}^{x + \delta} t^{{1 \over k} - 1} dt \le {\delta \over k}  x^{{1 \over k} - 1}.
$$
Using the previous inequality we obtain the bound:
$$
	f_k(m_k + \sqrt{ 32 m_k \ln(K/\delta) / I} ) \le  \lambda_2 + {\cal E}_{k}.
$$
Therefore, if event ${\cal A}$ occurs then:
$$
	\lambda_2 \le {\hat \ell}_\star \le  \lambda_2 + \min_{k=1,\dots,K}  {\cal E}_{k},
$$
We immediately deduce the first statement.
$$
	\PP\left( \lambda_2 \le {\hat \ell}_\star \le \lambda_2 + \min_{k=1,\dots,K}  {\cal E}_{k} \right) \ge  \PP({\cal A}) \ge 1 - \delta,
$$
since we have previously proven that $1 - \PP({\cal A}) \le \delta$. The second statement follows from the fact that $|{\hat \ell}_\star - \lambda_2 |$ is upper bounded by $\lambda_2 + \min_{k=1,\dots,K}  {\cal E}_{k}$ if ${\cal A}$ occurs and by $1$ otherwise, and the fact that $\PP({\cal A}) \ge 1 -\delta$. \qed

\section{Proof of proposition~\ref{prop:scaling}}
{\bf Proof:}
	Recall that ${\cal E}_{k} = {\cal E}_{k,1} + {\cal E}_{k,2}$.	We first focus on ${\cal E}_{k,1}$. Since $1 = \lambda_1 \ge \dots \ge \lambda_{|\Omega|}$, we upper bound $m_k$ as:
	$$
		m_k = {1 \over |\Omega|} \sum_{i=1}^{|\Omega|} \lambda_i^k \le {1 \over |\Omega|}(1 + \lambda_2^k + |\Omega| \lambda_3^k).
	$$
	Therefore:
	$$
		(|\Omega|m_k  - 1)^{1 \over k} \le (\lambda_2^k + |\Omega| \lambda_3^k)^{1 \over k} \le \lambda_2 + {\lambda_2 |\Omega| \over k} \left({\lambda_3 \over \lambda_2}\right)^k,
	$$	
	using once again the inequality $(x + \delta)^{1 \over k} \le x^{1 \over k} + {\delta \over k}  x^{{1 \over k} - 1}$. Hence:
	$$
		{\cal E}_k^1 \le {|\Omega| \over k} \left({\lambda_3 \over \lambda_2}\right)^k.
	$$
	We turn to ${\cal E}_k^2$. Recall that $m_k \ge {1 + \lambda_2^k \over |\Omega|}$, so that 
	$$
		(m_k |\Omega| - 1)^{1 - {1 \over k}} \ge \lambda_2^{k-1}.
	$$
	Using both the fact that $(m_k |\Omega| - 1)^{1 - {1 \over k}} \ge \lambda_2^{k-1}$ and that:
	$$
		m_k = {1 \over |\Omega|} \sum_{i=1}^{|\Omega|} \lambda_i^k \le {1 \over |\Omega|}(1 + (|\Omega| - 1)\lambda_2^k) \le {1 \over |\Omega|} + \lambda_2^k. 
	$$
	 we get the upper bound:
	$$
		{\cal E}_k^2 \le {|\Omega| \sqrt{ 32  ({1 \over |\Omega|} + \lambda_2^k) \ln(K/\delta) } \over \sqrt{I} k \lambda_2^{k-1}}
	$$	
	Since ${\cal E}_k = {\cal E}_{k,1} + {\cal E}_{k,2}$ we have proven that:
	$$
		{\cal E}_k  \le {|\Omega| \over k} \left( \left({\lambda_3 \over \lambda_2}\right)^k + {\sqrt{ 32  ({1 \over |\Omega|} + \lambda_2^k) \ln(K/\delta) } \over \sqrt{I} \lambda_2^{k-1}} \right).
	$$
	It remains to upper bound $\min_{k=1,\dots,K}{\cal E}_k$ and to do so by considering $k = k^\star(K,I,\delta)$. Since $\lambda_2^{k^\star(K,I,\delta)} \le {1 \over |\Omega|}$ we have that:
	$$
		{\cal E}_{k^\star(K,I,\delta)} \le {|\Omega| \over k^\star(K,I,\delta)} \left( \left({\lambda_3 \over \lambda_2}\right)^{k^\star(K,I,\delta)} + \vspace{-0.1cm} {\Delta(K,I,\delta)^{1 \over 2} \over \lambda_2^{k^\star(K,I,\delta)}} \right)
	$$
	In fact both terms in the r.h.s. equal, since:
	$$
		k^\star(K,I,\delta) = {\ln(1/\Delta(K,I,\delta)) \over 2 \ln(1/\lambda_3)}. 
	$$
	so that $\lambda_2^{k^\star(K,I,\delta)} = \Delta(K,I,\delta)^{r \over 2}$ and:
	\begin{align*}
		{\cal E}_{k^\star(K,I,\delta)} &= {2 |\Omega| \over k^\star(K,I,\delta)} \Delta(K,I,\delta)^{1 - r \over 2} 
			= {4 |\Omega| \ln(1/\lambda_3) \over \ln({ |\Omega| I \over 64 \ln(K/\delta)})}  \left( 64 \ln(K/\delta)  \over |\Omega| I\right)^{{1 - r \over 2}}.
	\end{align*}
	We obtain the announced inequality by noticing $\min_{k=1,\dots,K} {\cal E}_k \le 	{\cal E}_{k^\star(K,I,\delta)}$ since $k^\star(K,I,\delta) \le K$.

\section{Proof of proposition~\ref{prop:convergence_rate}}
	To prove the result, we first check that the conditions of proposition~\ref{prop:scaling} hold. We have:
 $$
 	\Delta(K,I,\delta) = { 64 \ln(K/\delta)  \over |\Omega| I} = {64 (\ln n)^2 (2 \ln(n) + {1 \over 2} \ln \ln(n)) \over |\Omega| n}.
 $$
 We have $\Delta(K,I,\delta) \ge {1 \over |\Omega| n}$ so that: 
 $$
 	k^\star(K,I,\delta) = {\ln(1/\Delta(K,I,\delta)) \over 2 \ln(1/\lambda_3)} \le { \ln |\Omega| + \ln n  \over  2 \ln(1/\lambda_3)}.
 $$
 Since $K = (\ln n)^2$, the above implies $k^\star(K,I,\delta) \le K$ for $n$ large enough. Furthermore, since $k^\star(K,I,\delta) \to \infty$, when $n \to \infty$ and $\lambda_2 < 1$, we have $\lambda_2^{k^\star(K,I,\delta)} \le {1 \over |\Omega|}$ for $n$ large enough. Hence the conditions of proposition 2 hold for $n$ large enough, and we have:
 \begin{align*}
 	\min_{k=1,\dots,K} {\cal E}_k &\le {4 |\Omega| \ln(1/\lambda_3) \over \ln( 1/\Delta(K,I,\delta))} \left( \Delta(K,I,\delta) \right)^{{1 - r \over 2}} 
 	\le {4 |\Omega| \ln(1/\lambda_3) \over \ln( {|\Omega| n \over 160 (\ln n)^3}  )} \left( {160 (\ln n)^3  \over |\Omega| n} \right)^{{1 - r \over 2}} 
 \end{align*}
 Since, from above $\Delta(K,I,\delta) \le {160 (\ln n)^3  \over |\Omega| n}$. Using the fact that $\EE |\hat{\ell}_\star - \lambda_2| \le \delta + \min_{k=1,\dots,K} {\cal E}_k$ concludes the proof. \qed


\end{document}